%% file: arxiv_CRCCBF_IROS2025.tex
\documentclass[letterpaper, 10 pt,dvipsnames, conference]{ieeeconf}  

\IEEEoverridecommandlockouts                              

\usepackage{amsmath}
\usepackage{amssymb}
\usepackage{amsfonts}
\usepackage{algorithm}
\usepackage[noend]{algpseudocode}
\usepackage{comment}
\usepackage{subcaption}
\usepackage{graphicx}
\usepackage{booktabs}
\usepackage{multirow}
\usepackage{colortbl}
\let\labelindent\undefined
\usepackage{enumitem}
\usepackage{thm-restate}
\definecolor{lightblue}{rgb}{0.85,0.91,0.95}
\newcommand{\mc}{\mathcal}
\newcommand{\mb}{\mathbb}
\usepackage{cite}
\usepackage{graphicx}
\usepackage{float}

\usepackage[colorlinks=true,linkcolor=blue,citecolor=blue]{hyperref} 

\usepackage{graphics} 
\usepackage{epsfig} 
\newtheorem{definition}{Definition}

\newtheorem{lemma}{Lemma}

\newtheorem{remark}{Remark}
\newtheorem{assumption}{Assumption}

\newcommand{\Lip}{\mathfrak{L}}

\usepackage{xcolor}

\input{macros}

\newcommand{\hum}{\mathrm{\tt H}}
\newcommand{\rob}{\mathrm{\tt R}}

\title{\LARGE \bf
Safe Probabilistic Planning for Human-Robot Interaction using Conformal Risk Control}

\author{Jake Gonzales$^{1}$ $\quad$ Kazuki Mizuta$^{2}$ $\quad$ Karen Leung$^{2*}$ $\quad$ Lillian J. Ratliff$^{1*}$%
\thanks{$^{1}$University of Washington, Department of Electrical and Computer Engineering, $^{2}$University of Washington, Department of Aeronautics and Astronautics.  $^*$Equal advising. JG and KM are supported by the UW + Amazon Science Hub Fellowship. KL was supported by a National Science Foundation award under Grant No. 2430686.}%
\thanks{Email: \texttt{\{jakegonz,mizuta,ratliffl,kymleung\}@uw.edu}}}

\begin{document}

\maketitle
\thispagestyle{empty}
\pagestyle{empty}

\begin{abstract}
In this paper, we present a novel probabilistic safe control framework for human-robot interaction that combines control barrier functions (CBFs) with conformal risk control to provide formal safety guarantees while considering complex human behavior. The approach uses conformal risk control to quantify and control the prediction errors in CBF safety values and establishes formal guarantees on the probability of constraint satisfaction during interaction. We introduce an algorithm that dynamically adjusts the safety margins produced by conformal risk control based on the current interaction context. Through experiments on human-robot navigation scenarios, we demonstrate that our approach significantly reduces collision rates and safety violations as compared to baseline methods while maintaining high success rates in goal-reaching tasks and efficient control. The code, simulations, and other supplementary material can be found on the project website: \url{https://jakeagonzales.github.io/crc-cbf-website/}. 

\end{abstract}


\section{Introduction} 
The safe deployment of autonomous robots in human environments, from autonomous driving to service robots, presents fundamental challenges in ensuring safety under the unpredictability of human behavior. In particular, the uncertainty in human behavior is \textit{multimodal}, where there could be multiple possible distinct behaviors (e.g., passing to the left or right of someone), and \textit{history dependent}, where future behaviors depend on interaction history.
Common approaches to safety-critical control in the presence of uncertainty make simplifying distributional assumptions (e.g. Gaussian) about the uncertainty and leverage chance constraints to provide probabilistic safety guarantees \cite{MoraChanceConstrained2019}. While sampling-based methods  can handle more complex distributions, they lack formal safety guarantees and can be computationally expensive for real-time applications \cite{Williams2015ModelPP}. 

A core challenge is in developing robot planning algorithms that can handle the complex uncertainty associated with human behavior without being overly conservative and 
providing quantifiable confidence in constraint satisfaction.
\begin{figure}[t]
    \centering
    \includegraphics[width=0.9\columnwidth]{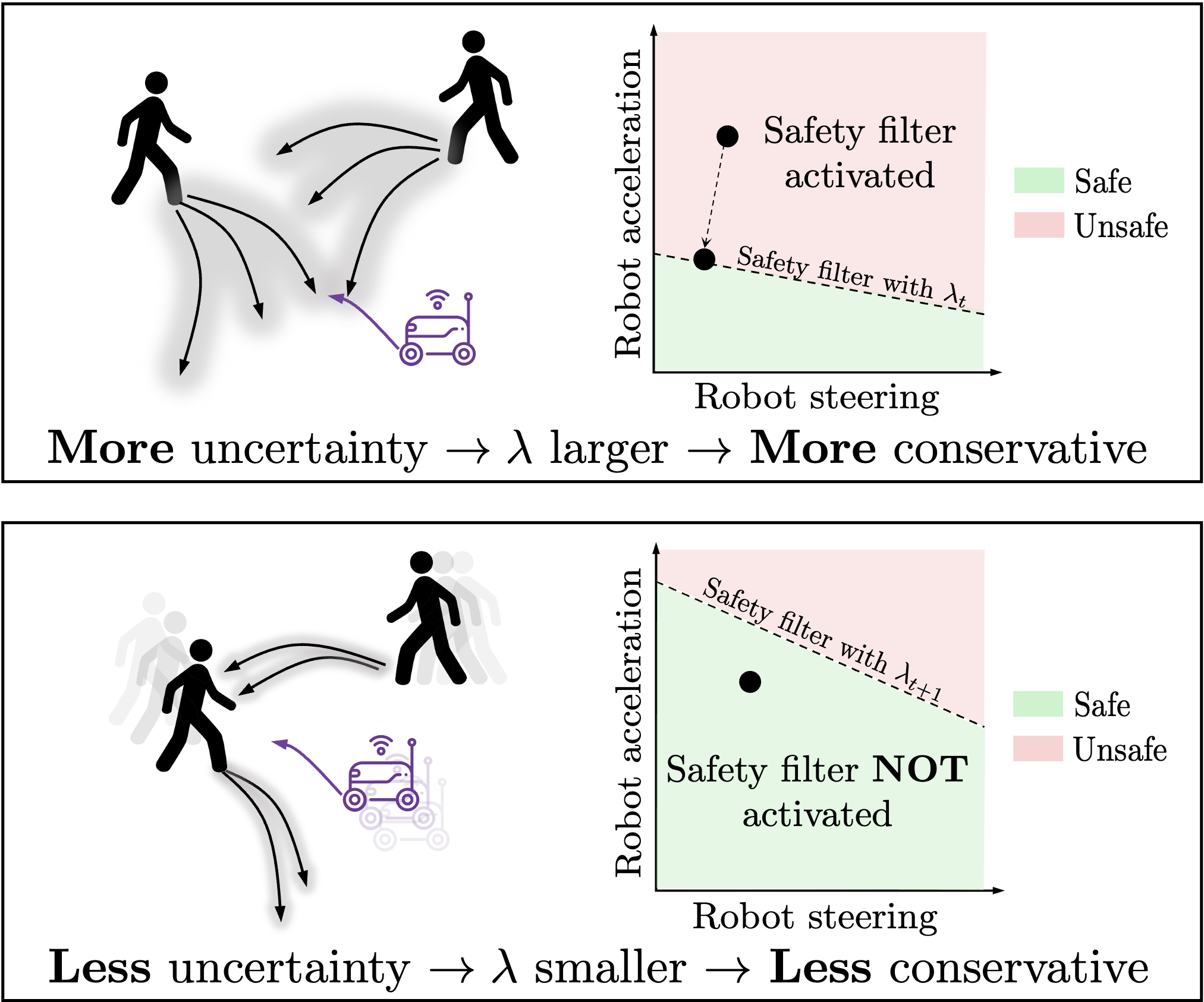}
    \caption{\small We develop a risk-aware adaptive safety filter that dynamically adjusts robot conservativeness based on the uncertainty of human behavior. The plots show the robot's control space with red regions indicating unsafe actions. In high-risk scenarios (top), our approach increases the \textbf{safety margin parameter $\boldsymbol{\lambda}$}, creating more restrictive control constraints. In low-risk scenarios (bottom), $\lambda$ decreases, allowing less conservative behavior while maintaining safety guarantees.}
   \vspace{-6mm}
    \label{fig:overview}
\end{figure}
Recent theoretical advancements in statistical verification have provided powerful tools for establishing formal guarantees on complex predictive models, including deep neural networks. In particular, \emph{conformal prediction} \cite{CP} offers a distribution-free framework for constructing prediction sets with guaranteed finite-sample coverage properties, and has been extended to the adaptive setting to maintain coverage guarantees under arbitrary distribution shifts using adaptive conformal prediction \cite{Gibbs2022ACP}. This extension is important for robot decision-making in dynamic environments where adaptive conformal prediction allows you to recalibrate prediction sets online, without requiring stationarity assumptions. 
Conformal Risk Control (CRC) \cite{angelopoulos2023conformalriskcontrol} 
further generalizes this approach by directly controlling arbitrary prediction risks through guarantees on the expected value of user-defined loss functions, making it particularly suitable for safety-critical applications where precise risk quantification is essential.

In this work, we present a principled framework that combines conformal risk control with control-theoretic safety constraints to achieve provable safety guarantees in human-robot interactions. 
Specifically, we apply CRC to algorithmically tune a \textit{safety margin} for a safety constraint, expressed using Control Barrier Functions (CBFs). 
We devise a strategy to tune the safety margin to meet a user-desired probabilistic safety specification. We apply this technique to learn a model using offline data to estimate the safety margin conditioned on a given context, thereby resulting in robots that dynamically adjust their conservativeness based on uncertainty. Figure~\ref{fig:overview} illustrates this concept.
The practical implementation of our approach does not rely on any distributional assumptions or worst-case bounds, and is both computationally efficient for real-time control and theoretically grounded.

\noindent\textbf{Statement of contributions.} The contributions are multifold: \textbf{(i)} a novel safe control framework combining CBFs with CRC to provide high-probability safety guarantees for human-robot interactions; \textbf{(ii)} rigorous theoretical analysis establishing the connection between CRC-based uncertainty quantification and probabilistic safety guarantees; \textbf{(iii)} an assumption-light algorithm for dynamic uncertainty adaptation of safety margins;
and \textbf{(iv)} a practically implementable and empirically validated approach to handle temporal dependencies in human-robot interactions. To our knowledge, this is the first use of conformal risk control to quantify uncertainty in control-theoretic safety constraints for joint human-robot systems. 

\subsection{Related work}

\textbf{Safe planning with CBFs.} By enforcing control invariance, CBFs provide theoretical guarantees for safety-critical control \cite{Ames2019ControlBF}. Yet, the practical implementation of CBFs faces challenges. For example, manual construction of CBFs is increasingly difficult for high-dimensional systems with input constraints. This limitation has motivated learning-based approaches to approximate CBFs \cite{Saveriano2019, Lindemann2020LearningHC, LearningCBFsExpertDems2020, SoHow2Train2024}.

Although CBFs have shown promise in single-agent settings, their application to multi-agent scenarios---particularly with human agents---remains underdeveloped. The inherent stochasticity of human behavior and unknown policies of opponents in multi-agent settings presents fundamental challenges that classical CBF formulations struggle to adequately address.
Towards this end, CBFs have been extended to robust formulations  \cite{Da2023RobustCB, Xu2016RobustnessOC, CosnerCulbertsonEtAl2023} for \emph{a priori} known bounded disturbances, however this is not practical for real-world deployment with complex and unbounded uncertainties due to human interactions or other autonomous agents. 

\textbf{Safe planning with conformal prediction.} To address uncertainties in multi-agent settings with minimal assumptions, there has been a line of work on \emph{conformal prediction} in safety critical control. 
For instance, ignoring feedback effects, recent work \cite{LindemannSafePlanningCP2023} uses conformal prediction in a model predictive control formulation to model opponent uncertainty. In a follow-up work \cite{dixit2022adaptiveconformalpredictionmotion}, the authors leveraged ideas from  adaptive conformal prediction to adapt to distribution shift by quantifying uncertainty from past agent observations without requiring stationarity assumptions on agent motion. More closely aligned with our approach, \cite{zhou2024safetycriticalcontroluncertaintyquantification} integrates adaptive conformal prediction  with CBF constraints in order to safely adapt to uncertainty in multi-robot systems. 

Conformal prediction methods create confidence regions around other agents' trajectories by calibrating prediction sets for coverage.
As such they lack the ability to directly control the \emph{risk} of safety constraint violations. Instead of targeting fixed coverage, CRC aims to control the expected loss (risk) under a user-specified loss function and potentially adaptive or contextual thresholds.
Indeed, our approach addresses this limitation by leveraging CRC to explicitly quantify prediction errors in safety values consequent of uncertainty due to other agent behavior, and encoding this directly into the safety constraints. This
provides a more precise and principled approach to managing the trade-off between safety and performance.
Parallel work by \cite{huriot2024safedecentralizedmultiagentcontrol} proposes a similar formulation that uses conformal decision theory \cite{ConformalDecisionTheory2024} to tune the conservativeness of the CBF constraint. A limitation of this approach, however, is that it only provides guarantees on the \emph{average constraint violation} over time rather than probabilistic safety guarantees for each timestep, arguably a crucial need for safety-critical control involving humans. We establish this theoretical connection and empirically evaluate our methods using
human behavior models trained on real-world pedestrian data. 

\section{Preliminaries}

Consider a human-robot joint dynamical system where the goal is to design a control policy for the robot such that the joint state $x(t)\in \mb{R}^n$ of the dynamical system remains safe  with high probability given the stochastic unknown policies of all other agents.

Let $\rob$ denote the index of the robot agent, and $\hum=\{1,\ldots, H\}$ be the index set of all agents excluding the robot agent. 
For any time $t\in \mb{R}_+$, define the joint state $x(t):=(x_{\rob}(t),x_{\hum}(t))$ where $x_{\rob}(t) \in \mc{X}_\rob \subseteq \mathbb{R}^{n_\rob}$ and $x_{\hum}(t) \in\mc{X}_\hum \subseteq \mathbb{R}^{n_\hum}$, denote the robot and human states, respectively, and let $\mc{X}:=\mc{X}_{\rob}\times\mc{X}_{\hum}$ define the joint state space. Analogously, 
 define the joint control input $u(t):=(u_{\rob}(t),u_{\hum}(t))$ where $u_{\rob}(t) \in \mc{U}_\rob\subseteq \mathbb{R}^{m_\rob}$ and $u_{\hum}(t) \in \mathbb{R}^{m_\hum}$ denote the robot and human control inputs, respectively, and let $\mc{U}:=\mc{U}_{\rob}\times\mc{U}_{\hum}$ be the joint control space. Note that we drop the explicit dependence on $t$ where obvious from context. 
 \begin{assumption}\label{a:u_conditions}
     The joint control space $\mc{U}$ is non-empty, closed, and convex set. Moreover,  there exists strictly positive, finite constants $b_u$ and $B_u$ such that $b_u\mb{B}\subseteq \mc{U}\subseteq B_u\mb{B}$ where $\mb{B}=\{u\in \mb{R}^m\mid \|u\|=1\}$ is the unit ball.  
 \end{assumption}
 This assumption implies the set $\mc{U}$ is compact with a non-empty interior; otherwise, we can map $\mc{U}$ to a lower dimensional space.
 Let $n:=n_{\rob}+n_{\hum}$ be the state dimension and $m:=m_{\rob}+m_{\hum}$ be the control input dimension. 
The dynamical system under consideration is given by
\begin{equation}
\begin{aligned}
      \dot{x} &=f(x)+Bu,
\end{aligned}
    \label{eq:continuous_joint_dynamics}
\end{equation}
where $B= \bmat{B_{\rob} & B_{\hum}}$. 
We make the following regularity assumption.
\begin{assumption}\label{a:lip_continuous_dynamics}
    The map  $f: \mc{X} \rightarrow \mb{R}^{n}$ is locally Lipschitz continuous in $x$ so that $x(t)$ is Lipschitz continuous.
\end{assumption}
Let  $\Lip_x\in[0,\infty)$ be the corresponding Lipschitz constant.

\subsection{Safety}
As noted the goal is to give a guarantee on the state remaining safe. To gain some intuition, let us consider a deterministic 
control-affine system: 
\begin{equation}
\dot{x} = f(x) + g(x)u
\label{eq:control_affine}
\end{equation}
where $x \in \mc{X} \subset \mathbb{R}^n$ is the state, and $u \in \mc \mc{U} \subset \mathbb{R}^m$ is the control input. 
Let $\mc{S} \subseteq \mc{X}$ denote the set of \emph{safe states}  for the system in \eqref{eq:control_affine}.
One way to characterize such a set is via the $0$-superlevel set of a mapping: 
\[\mc{S} := \{x \in \mc{X} \mid h(x) \geq 0\},\]
where $h(\cdot)$ is a continuously differentiable function. 

In this setting, the objective is to synthesize a feedback control policy $\pi: x \mapsto u$ ensuring forward invariance of $\mc{S}$ under closed-loop dynamics:
\[\dot{x}=f(x)+g(x)\pi(x).\]
That is, trajectories initialized in $\mc{S}$ at time $t_0$ must remain within $\mc{S}$ for all $t \geq t_0$. This is known as the safe planning problem. One common approach to solving this problem is to use CBFs \cite{Ames2019ControlBF}.

\begin{definition} 
The mapping $h(\cdot)$ is a CBF for the system \eqref{eq:control_affine} if there exists an extended class-$\mc{K}_{\infty}$\footnote{An extended class--$\mc{K}_{\infty}$ function $\mc{K} : \mb{R} \rightarrow \mb{R}$ is continuous, strictly increasing, satisfies $\mc{K}(0) = 0$, and $\mc{K}(r) \to \infty$ as $r \to \infty$.} function $\mc{K}$ such that
\begin{equation}
\sup_{u\in \mc{U}}\{L_f h(x) + L_g h(x)u + \mc{K}(h(x))\} \geq 0, \quad \forall x \in \mc{X},
\end{equation}
where $L_f h(x) = \nabla h(x)^\top f(x)$ and $L_g h(x) = \nabla h(x)^\top g(x)$.
\end{definition}
For a given CBF
$h(\cdot)$ and class-$\mc{K}_{\infty}$ function $\mc{K}$,  the set of admissible safe control actions is defined  as
\begin{equation}
C(x) := \{u \in \mc{U} \mid L_f h(x) + L_g h(x)u + \mc{K}(h(x)) \geq 0\}.
\label{eq:CBF_safe_control_set}
\end{equation}
Any locally Lipschitz continuous controller $\pi(x) \in C(x)$ guarantees forward invariance of the safe set $\mc{S}$ under the closed-loop dynamics. This motivates the design of a minimally-invasive safety filter:
\begin{align}
\pi(x) := \min_{u \in \mc{U}} \quad &  \|u - u_{\text{nom}}(x)\|_2^2 \label{eq:safety_filter_main}\tag{\tt CBF-QP}\\
\text{subject to} \quad & L_f h(x) + L_g h(x)u + \mc{K}(h(x)) \geq 0, \notag
\end{align}
where $u_{\text{nom}}(x)$ represents a control input from a potentially unsafe nominal controller. Given the control-affine structure of our system, $C(x)$ forms a polyhedral set, reducing \eqref{eq:safety_filter_main} to a convex quadratic program.

In the setting we consider, however, there are two additional practically motivated challenges: 
\begin{itemize}[itemsep=0pt,topsep=2pt,leftmargin=10pt]
    \item \textbf{Discretization}: Agent controllers generally operate in discrete time due to sampling. Therefore we must discretize the dynamics and synthesize a robust safe policy. 
    \item \textbf{Uncertainty}: The human's behavior is unknown \emph{a priori} to the robot. This means they do not have a deterministic model of the control input $u_{\hum}$, and must learn it from data. 
\end{itemize}

\begin{figure*}[t]
    \centering
    \begin{subfigure}[b]{0.34\textwidth}
        \centering
        \includegraphics[width=\textwidth]{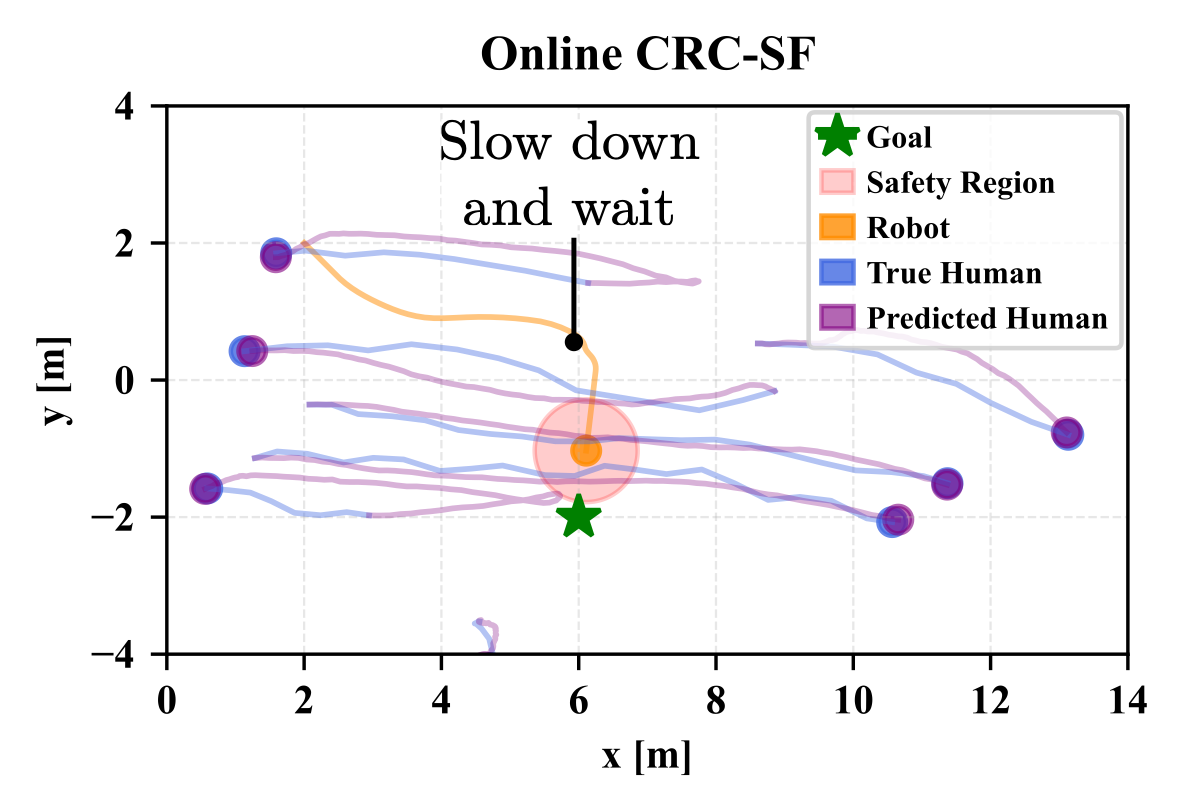}
    \end{subfigure}
    \hspace{-1.0em}
    \begin{subfigure}[b]{0.34\textwidth}
        \centering
        \includegraphics[width=\textwidth]{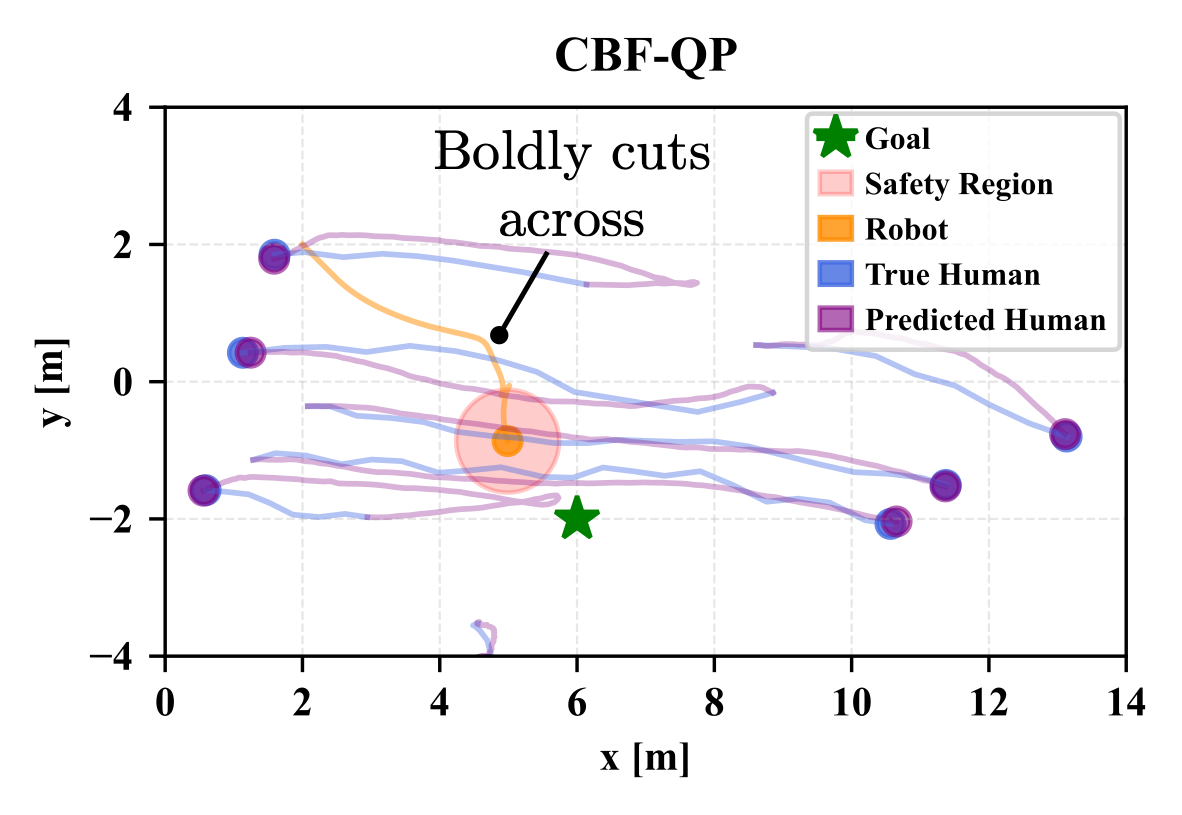}
    \end{subfigure}
    \hspace{-1.0em}
    \begin{subfigure}[b]{0.34\textwidth}
        \centering
        \includegraphics[width=\textwidth]{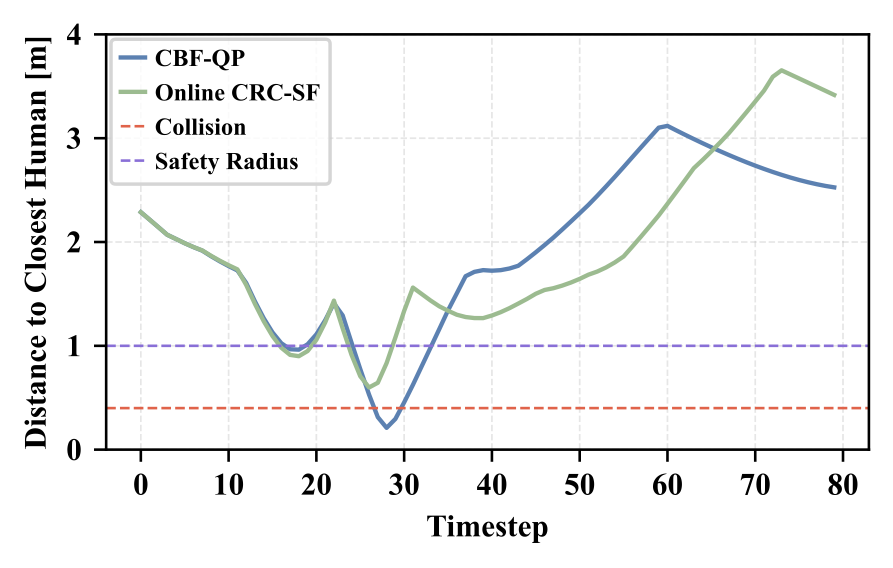}
    \end{subfigure}
    \hspace{-1.0em}
    \caption{ \small \textbf{Multi-Agent Scenario.} Example of the human crowd simulation setup. The CRC Safety Filter maintains the probabilistic safety guarantees while the standard CBF exhibits unsafe behavior. The safety margin $\lambda$ adapts based on human prediction uncertainty, ensuring larger distances are maintained during uncertain interactions.}
    \vspace{-1em}
    \label{fig:results_comparison}
\end{figure*}

\subsection{Discretization \& Robust Control Barrier Function}

 We discretize the continuous-time dynamics given in \eqref{eq:continuous_joint_dynamics}. 
Let $t_k = k \cdot \Delta T$ for $k \in \mb{N}$ denote sampling instants with fixed timestep $\Delta T \in \mb{R}_{>0}$. Under this time discretization, the controller operates using a zero-order hold control law: for each $k\in \mb{N}$, set
\begin{equation}
    u(t) = u_k, \quad \forall t \in [t_k, t_{k+1}).
    \label{eq:zoh}
\end{equation}
 Let $x_k:=(x_{\rob}(t_k),x_{\hum}(t_k))$ be the corresponding sampled state\footnote{Under a zero-order hold control law as in \eqref{eq:zoh} with compact set $\mc{U}$, uniqueness of the maximal closed-loop solution $x_k$ is guaranteed \cite{Sontag1998MathematicalCT}.}. Thus, discrete-time joint dynamics are give by
\begin{equation}
    x_{k+1} = f_d(x_{k}) + B_{\rob,d} u_{\rob,k} + B_{\hum,d} u_{\hum,k} :=F_d(x_k,u_k),
    \label{eq:joint_dynamics}
\end{equation}
where $B_{j,d}:=B_j\cdot \Delta T$ for $j\in\{\rob,\hum\}$. The discrete-time dynamics $f_d(x_k)$ are obtained through standard numerical integration of $f(x(t))$ over the time interval $[t_k, t_{k+1}]$ under the zero-order hold assumption on the control inputs. The discrete-time dynamics $f_d(x_k)$ are implicitly defined by this integral using  Gronwall's inequality and the Lipschitz continuity of $f$. 

To ensure safety despite the time discretization, we construct a \emph{robust CBF} \cite{pmlr-v155-dean21a} and corresponding safe controller set. We show first in the deterministic setting that the zero-order hold controller remains safe on every interval $[t_k,t_{k+1}]$ as long as $u_k\in \mc{C}(x_k)$ and $x_k\in \mc{S}$. 
To this end, we make the following regularity assumptions.
 \begin{assumption}\label{a:regularity} The mappings  $h\in C^1$  and $\mc{K} \circ h$  are  $\Lip_h$  and $\Lip_{\mc{K}\circ h}$ Lipschitz continuous, respectively. 
 \end{assumption}
Observe that $L_f h(x)$ and $L_gh(x)$ are, respectively,  $\Lip_h\Lip_f$ and $\Lip_h\Lip_g$ Lipschitz continuous under Assumption~\ref{a:lip_continuous_dynamics}. 

Now, define the robust safe control set---which we will use to select the control at each time $t_k$---as follows: for the current state and agent action $(x(t_k), u_{\hum}(t_k))$, the robust safe set for the robot is given by
\begin{equation*}
\begin{aligned}
  \widetilde{\mc{C}}(x(t_k), u_{\hum}(t_k)) &= \{u_{\rob} \in \mc{U}_{\rob} \mid  L_fh(x(t_k)) +L_gh(x(t_k))u \\&+ \mc{K}(h(x(t_k))) \geq \eta, \; u=(u_{\rob},u_{\hum}(t_k)) \}
\end{aligned}
\end{equation*}
where $\eta := \Delta T \cdot \Lip_x(\Lip_h\Lip_{g} R_u+\Lip_h\Lip_{f}+\Lip_{\mc{K}\circ h})$ is a robustness margin accounting for discretization.

\begin{restatable}[Robust CBF]{lemma}{robustcbflemma}
\label{lem:robust_cbf_lemma}
Consider the interval $[t_k,t_{k+1}]$ where $\Delta T=t_{k+1}-t_k$. Suppose that the agent is employing a zero-order hold policy, and that at each discrete time step $k$, the proposed algorithm selects $u_{\rob,k}\in \widetilde{\mc{C}}(x(t_k), u_{\hum}(t_k))$ for the current state and agent action. 
Further, suppose that the robot uses a fixed control input $u_{\rob}(t)\equiv u_{\rob, k}$ on $[t_k,t_{k+1}]$ and that the current state is safe---i.e., $x(t_k)\in \mc{S}$.  Then $x(t)\in \mc{S}$ for all $t\in [t_k,t_{k+1}]$. 
\end{restatable}
The proof of the preceding lemma is given in Appendix~\ref{subsec:proof_robust_cbf}. Observe that this lemma provides a safety guarantee for the entire discretized time interval.

In the deterministic setting---i.e., where the human behavior policy is known to the robot---we could easily modify the optimization problem in \eqref{eq:safety_filter_main} by replacing the control barrier function constraint with the robust version defined above: namely,
\begin{equation}
\begin{aligned}
\pi(x(t_k)) := \min_{u_{\rob}} \quad &  \|u_{\rob} - u_{\text{nom}}(x(t_k))\|_2^2 \\
\text{subject to} \quad & u_{\rob} \in  \widetilde{\mc{C}}(x(t_k), u_{\hum}(t_k))
\end{aligned}\tag{\tt RCBF-QP}
\label{eq:safety_filter_main_robust}
\end{equation}
However, in practice, such knowledge would not be available to the robot. 
\subsection{Stochastic Human Behavior with Probabilistic Safety}
\label{subsec:human_behavior_model}
Now that we have tackled the time discretization, we can deal with the second challenge on the stochasticity of the human policy. 
As noted above, the human's policy is unknown \emph{a priori}.
Indeed, since the human's actions are stochastic, we cannot optimize \eqref{eq:safety_filter_main_robust} since the robot does not have $u_{\hum}(t_k)$. Instead, we aim to ensure this condition holds with high probability. 

The following lemma establishes a probabilistic safety guarantee by imposing constraints on the robots control actions given a stochastic policy $P(u_{\hum}(t_k)|\ x_{0:k})$ for the human agents, where $x_{0:k}:=(x(t_0), \ldots, x(t_k))$.  
\begin{lemma}
\label{lem:prob_safety_guarantee}
Suppose the joint control space $\mc{U}$ satisfies Assumption \ref{a:u_conditions},  the system \eqref{eq:joint_dynamics} satisfies Assumptions~\ref{a:lip_continuous_dynamics} and \ref{a:regularity}, and both the robot and humans employ zero-order hold control policies.  
Fix a confidence level $\gamma \in (0,1)$.
Suppose $x(t_k) \in \mc{S}$ and the robot selects $u_{\rob}(t_k)$ such that 
\begin{equation}
\Pr(u_{\rob}(t_k) \in \widetilde{\mc{C}}(x(t_k), u_{\hum}(t_k))) \geq 1 - \gamma,
\end{equation}
where the probability is taken with respect to the distribution of human actions $u_{\hum}(t_k)\sim P(u_{\hum}(t_k)|\ x_{0:k})$. Then, the joint state remains safe for all $t\in[t_k,t_{k+1}]$---i.e., 
\begin{equation}
\Pr\left(\{x(t) \in \mc{S}\;\; \forall t \in [t_k, t_{k+1}]\}\right) \geq 1 - \gamma.
\end{equation}
\end{lemma}
This lemma trivially follows from Lemma~\ref{lem:robust_cbf_lemma} which implies that if $u_{\rob}(t_k) \in \widetilde{\mc{C}}(x(t_k), u_{\hum}(t_k))$ and $x(t_k) \in \mc{S}$, then $x(t) \in \mc{S}$ for all $t \in [t_k, t_{k+1}]$.

Given the safety guarantee implied by Lemma~\ref{lem:prob_safety_guarantee}, we modify the \eqref{eq:safety_filter_main_robust} to incorporate this high probability guarantee: namely, 
\begin{align}
\pi(x(t_k)) := \min_{u_{\rob}} \quad &  \|u_{\rob} - u_{\text{nom}}(x(t_k))\|_2^2 \tag{\tt SRCBF-QP}
\label{eq:safety_filter_main_robust_safe}\\
\text{subject to} \quad & \Pr(u_{\rob}(t_k) \in \widetilde{\mc{C}}(x(t_k), u_{\hum}(t_k))) \geq 1 - \gamma\notag
\end{align}
This optimization problem now has a chance constraint. 

\section{Robust and Safe Control Policy Synthesis}
\label{sec:proposed_solution}
In this section, we describe a method to address the probabilistic safety constraint in \eqref{eq:safety_filter_main_robust_safe}. We reformulate the CBF condition using CRC to quantify the risk associated with the uncertainty in human's stochastic policy. 

\subsection{Safety Barrier Certificates}
Following \cite{LiSafetyBarrierCertificates2017}, we define two safety barrier certificates, the deterministic $\mc{B}(x_k, u_k)$ using ground truth states and the predicted $\mc{\hat{B}}(\hat{x}_k, \hat{u}_k)$ using estimated states:
\begin{equation}
\begin{aligned}
\mc{B}(x_k, u_k) &= L_f h(x_k) + L_g h(x_k)u_k + \mc{K}(h(x_k)) - \eta, \\
\mc{\hat{B}}(\hat{x}_k, \hat{u}_k) &= L_f h(\hat{x}_k) + L_g h(\hat{x}_k)\hat{u}_k + \mc{K}(h(\hat{x}_k)) - \eta,
\end{aligned}
\label{eq:barrier_certificates}
\end{equation}
where $u_k = (u_{\rob,k}, u_{\hum,k})$ and $\hat{u}_k = (u_{\rob,k}, \hat{u}_{\hum,k})$. For safety, we require $\mc{B}(x_k, u_k) \geq 0$, but can only compute $\mc{\hat{B}}(\hat{x}_k, \hat{u}_k)$ during planning.

\begin{assumption}[Bounded barrier prediction error]
\label{a:bounded_barriers}
There exists a finite constant $\bdd \in \mathbb{R}_{>0}$ such that, for all time steps $k \in \mathbb{N}$, the difference between the true barrier certificate $\mc{B}(x_k, u_k)$ and the predicted barrier certificate $\mc{\hat{B}}(\hat{x}_k, \hat{u}_k)$ is bounded
$\sup |\mc{B}(x_k, u_k) - \mc{\hat{B}}(\hat{x}_k, \hat{u}_k)| \leq \bdd$. This bound holds for all joint states at each timestep $k$.
\end{assumption}
\begin{remark}
    While this assumption is used in the theoretical analysis to establish statistically valid guarantees, the proposed method does not require explicit knowledge of this bound in practice since CRC implicitly learns the appropriate margins from data. 
\end{remark}

\subsection{Safety Margin Synthesis via Conformal Risk Control}
To address the chance constraint in \eqref{eq:safety_filter_main_robust_safe},
we leverage nonexchangeable CRC \cite{farinhas2024nonexchangeable}, a statistical verification technique that provides a flexible framework for quantifying prediction errors in dynamic settings. Theoretical guarantees for classical conformal uncertainty quantification methods hold under the assumption that the data is exchangeable\footnote{Exchangeability is a slight relaxation of the well-known condition of independent and identically distributed data.} \cite{Vovk}. Data from controlled dynamical systems is not exchangeable since the current state depends on the past state and control inputs.  Non-exchangeable CRC accounts for these dependencies by essentially reweighting the conformal prediction sets. 
For brevity, we will refer to nonexchangeable CRC simply as \emph{CRC} from here on out. 

To produce safety guarantees using the estimated barrier certificate, we need to bound the estimation error. The procedure we develop to do this uses a \emph{safety margin} parameter that is selected based on the historical estimates of the safety barrier certificates as we describe below.

Let $\lambda \in \Lambda \subseteq \mb{R}_{>0}$ be a \emph{safety margin} parameter for constructing $C_\lambda$ that characterizes the set of safe robot control inputs under uncertainty: 
\begin{equation}
    C_{\lambda}(\hat{x}_k, \hat{u}_{\hum, k}) = \{u_{\rob, k} \in \mc{U}_\rob \mid \mc{\hat{B}}(\hat{x}_k, \hat{u}_k) - \lambda_k \geq 0\}.
    \label{eq:prediction_set}
\end{equation}

For any $\lambda, \lambda' \in \Lambda$, we have $\lambda \leq \lambda' \implies C_{\lambda'}(\cdot) \subseteq C_{\lambda}(\cdot)$, \textit{i.e.,} larger $\lambda$ yields a more restrictive control space. The safety margin essentially creates a buffer that accounts for prediction uncertainty: as uncertainty increases, $\lambda$ increases to ensure safety by further restricting the set of permissible controls, directly capturing the risk associated with the human's stochastic policy.

\subsubsection{Conformal Risk Control Safety Filter}
Now, we present a risk-aware safety filter that enforces probabilistic safety constraints through a quadratic program formulation. At each discrete timestep $k$, the CRC safety filter computes a control input that minimally deviates from a nominal robot policy while ensuring safety with high probability:
\begin{align}
\pi(x(t_k)) := \min_{u_{\rob}} \quad &  \|u_{\rob} - u_{\text{nom}}(x(t_k))\|_2^2 \tag{\tt CRC-SF}\label{eq:crc_safety_filter}\\
\text{subject to} \quad & \Pr(u_{\rob}(t_k) \in C_{\lambda}(\hat{x}_k, \hat{u}_{\hum, k})) \geq 1 - \gamma\notag
\end{align}
where $u_{\text{nom}}(x(t_k))$ is the nominal robot control and $C_{\lambda}(\hat{x}_k, \hat{u}_{\hum,k})$ is the set of safe robot controls parameterized by safety margin $\lambda$. We develop a method to update $\lambda$ online based on interaction context (Section \ref{sec:alg_details}). This formulation provides single-timestep safety guarantees under human behavioral uncertainty (Theorem~\ref{thm:crc_cbf}).

\subsubsection{Optimizing the Safety Margin Parameter}
To select the safety margin parameter, we define a loss function  $\mc{L} : \Lambda \rightarrow [0, B]$ that quantifies the barrier prediction error: 
\begin{equation}
\begin{aligned}
    \mc{L}(\lambda) = \max\{0, |\mc{B}(x_k, u_k) - \mc{\hat{B}}(\hat{x}_k, \hat{u}_k)| - \lambda \}.
    \label{eq:crc_loss}
\end{aligned}
\end{equation}
This loss function is bounded under Assumption~\ref{a:bounded_barriers}; this is a technical condition that is needed for the theoretical analysis of the CRC method. Further, this loss function quantifies the amount by which the prediction error exceeds the safety margin $\lambda$, and satisfies the key properties required by the CRC framework: (i) non-negativity, (ii) boundedness by Assumption~\ref{a:bounded_barriers}, and (iii) monotonically non-increasing with respect to $\lambda$.

To account for temporal dependencies in human-robot interaction, we apply nonexchangeable CRC using geometrically decaying weights. The optimal safety margin $\hat{\lambda}_k$ is chosen as the smallest value that ensures the weighted risk remains below $\alpha$ with an appropriate correction term:
\begin{equation}
\hat{\lambda}_k = \inf\left\{\lambda : \frac{1}{n_w + 1}\sum_{i=1}^{n_k} w_i \mc{L}_i(\lambda) + \frac{B}{n_w + 1} \leq \alpha\right\},
\label{eq:optimal_lambda}
\end{equation}
where $w_i = \rho^{n_k+1-i}$ are geometrically decaying weights with $\rho \in (0,1)$, $n_k$ is the size of the calibration set (detailed in Section~\ref{sec:alg_details}), and $n_w = \sum_{i=1}^{n_k} w_i$. This formulation ensures
\begin{equation}
   \mathbb{E}[\mc{L}(\hat{\lambda}_k)] \leq \alpha + \beta,
   \label{eq:nonx_crc_guarantee}
\end{equation}
where $\alpha \in (0,1)$ is a user-specified risk level. Here $\beta$ represents the \emph{total variation distance} that quantifies the degree of non-exchangeability in the data. As discussed in \cite{farinhas2024nonexchangeable}, in practice $\beta$ is difficult to compute, but with a proper choice of weights and gradual distribution shift such as that in time series data, $\beta$ is likely to be small.

\begin{algorithm}[t]
\caption{Offline Safety Margin Calibration using CRC}
\label{alg:compute_lambda}
\begin{algorithmic}[1]
\State \textbf{Input:} Risk threshold $\alpha$, human policy $P(u_{\hum}|\ x_{0:k})$, nominal robot policy $\pi_{\text{nom}}$, discount factor $\rho$, trajectories per batch $K$, total trajectories $M$, horizon $N$
\State \textbf{Initialize:} Training dataset $\mc{T} \leftarrow \emptyset$ 
\For{$b = 1$ to $M/K$} \Comment{Each batch uses $K$ trajectories to compute one set of $\lambda$ values}
   \State $\mc{D}_b \leftarrow \emptyset$ \Comment{Calibration data for batch $b$}
\For{$j = 1$ to $K$} \Comment{Interaction trajectories}
    \State Initialize states: $x_{\rob,0}, x_{\hum,0}$
    \For{$k = 0$ to $N-1$} \Comment{Data Collection}
        \State Sample human action: $\hat{u}_{\hum,k} \sim P(u_{\hum}|\ x_{0:k})$ 
        \State Set $u_{\rob,k} = \pi_\rob(x_k, \hat{u}_{\hum,k})$ solving \eqref{eq:safety_filter_main_robust}
        \State Update states: $x_{k+1} = F_d(x_k, u_k)$ via \eqref{eq:joint_dynamics}
        \State Compute barrier certificates as in \eqref{eq:barrier_certificates}
        \State $\mc{D}_b \leftarrow \mc{D}_b \cup \{(\mc{B}_k, \mc{\hat{B}}_k)\}$
    \EndFor
\EndFor
   \For{$k = 0$ to $N-1$} \Comment{Risk Calibration}
       \State $\mc{D}_{b,k} \leftarrow \{(\mc{B}_i, \mc{\hat{B}}_i, k_i) \in \mc{D}_b \mid k \leq k_i \leq N-1\}$ 
       \State $n_k \leftarrow |\mc{D}_{b,k}| = K \times (N-k)$ 
       \State $w_i \leftarrow \rho^{n_k+1-i}$ for $i = 1,\ldots,n_k$
       \State Compute $\hat{\lambda}_{b,k}$ using \eqref{eq:optimal_lambda}
       \State Extract feature vector $\phi_{b,k}$ for each $K$ trajectories
       \State $\mc{T} \leftarrow \mc{T} \cup \{(\phi_{b,k}, \hat{\lambda}_{b,k})\}$
   \EndFor
\EndFor
\State \Return $\mc{T}$
\end{algorithmic}
\end{algorithm}

\subsection{High Probability Safety Guarantees}
This formulation leads to the main theoretical results. First, we establish that CRC guarantees on the barrier prediction error translate to probabilistic bounds on the difference between true and predicted barrier values.

\begin{restatable}[Barrier Value Concentration]{lemma}{barriervalueconcentration}
\label{lem:crc_concentration}
Given a risk threshold $\alpha \in (0,1)$ and confidence level $\gamma \in (0,1)$, if we compute $\hat{\lambda}$ using non-exchangeable CRC to satisfy $\mathbb{E}[\mc{L}(\hat{\lambda})] \leq \alpha + \beta$, 
then setting $\epsilon = (\alpha + \beta)/\gamma$ gives
\begin{equation}
   \Pr(|\mc{B}(x_k, u_k) - \mc{\hat{B}}(\hat{x}_k, \hat{u}_k)| \leq \hat{\lambda}_k + \epsilon) \geq 1 - \gamma.
\end{equation}
\end{restatable}
The proof of Lemma~\ref{lem:crc_concentration} is contained in Appendix~\ref{subsec:proof_crc_concentration}.
Building on this result, we establish our safety guarantee.
\begin{restatable}[CRC-CBF Safety Guarantee]{theorem}{crccbfsafety}
\label{thm:crc_cbf}
Consider the human-robot system \eqref{eq:joint_dynamics} with barrier certificates defined in \eqref{eq:barrier_certificates}. Given a confidence level $\gamma \in (0,1)$ and risk threshold $\alpha \in (0,1)$, if we have the CRC guarantee such that $\hat{\lambda}$ satisfies $\mathbb{E}[\mc{L}(\hat{\lambda})] \leq \alpha + \beta$ and set $\epsilon = (\alpha + \beta)/\gamma$, then the prediction set defining the safe set of control inputs under uncertainty,
\begin{equation}
C_{\lambda}(\hat{x}_k, \hat{u}_{\hum, k}) = \{u_\rob \in \mc{U}_\rob \mid \mc{\hat{B}}(\hat{x}_k, \hat{u}_k) - (\hat{\lambda}_k + \epsilon) \geq 0\},
\label{eq:CBF prediction bound}
\end{equation}
ensures that $\Pr(h(x_{k+1}) \geq 0) \geq 1 - \gamma$ holds.
\end{restatable}
The proof of the above theorem can be found in Appendix~\ref{subsec:proof_crc_cbf}. This result shows that by bounding prediction errors on barrier functions using CRC, we can construct control constraints that guarantee safety with high probability.


\section{Algorithmic Details}
\label{sec:alg_details}
In this section, we present the algorithmic details for implementing the proposed probabilistic safety framework. We begin with a brief high-level overview of the procedure. The proposed method first trains a stochastic human behavioral policy $P(u_{\hum}|\ x_{0:k})$ on real-world pedestrian data. The interaction history $x_{0:k} = \{x_0, x_1, \ldots, x_k \}$ captures the joint state evolution of all agents in the scene, enabling the human policy to account for multi-agent dynamics through observed interactions. While our theoretical framework supports conditioning on robot control actions---through the development of robot-aware human behavior models as in \cite{Schmerling2017MultimodalPM}---our current implementation is trained on pedestrian datasets where robot actions are not present \cite{PellegriniEssEtAl2009}. Specifically, our human behavioral model takes as input the positions and velocities of all agents in the scene over the interaction history and outputs a sequence of predicted human controls, following approaches similar to \cite{SalzmannIvanovicEtAl2020}. The learned human policy captures relative spatial information (e.g., distance to nearest agents, relative velocities) that enables generalization from the pedestrian-only training data to human-robot interaction scenarios in our simulated environments.

The proposed method requires the specification of an \emph{a priori} nominal robot control policy $\pi_\text{nom}$ that governs robot behavior during the offline calibration phase. The nominal policy can be any reasonable control strategy chosen by the user---for instance, a goal-reaching controller or any existing deterministic safe control method. The nominal policy serves as the robot's baseline behavior around which we calibrate our uncertainty quantification and safety margin adaptation.

In an \emph{offline phase}, we use the learned human policy together with the nominal robot policy to generate interaction trajectories $\tau = \{ (x_k, u_k) \}_{k=0}^{N-1}$ that capture the joint evolution of human-robot state-action pairs over a finite horizon $N$.
For each trajectory, we compute both the true barrier certificate values $\mc{B}_k$ using ground truth human actions and the predicted values $\mc{\hat{B}}_k$ using actions sampled from our behavior model.
Then, we apply CRC to the collected barrier certificate pairs $\{ (\mc{B}_k, \mc{\hat{B}}_k ) \}$ to compute the optimal safety margins $\{ \hat{\lambda} \}_{k=0}^{N-1}$. We then train a predictive model that maps the current context in the scene to the appropriate safety margins, with details outlined later in this section. In the \emph{online phase}, we use this model to dynamically adapt $\lambda$, enabling risk-adaptive control.

\begin{algorithm}[t]
\caption{CRC Safety Filter with Online Updates}
\label{alg:crc_sf}
\begin{algorithmic}[1]
\State \textbf{Input:} Initial states $(x_{\rob,0}, x_{\hum,0})$, stochastic human policy $P(u_{\hum}|\ x_{0:k})$, LSTM model $\Phi$, timestep $\Delta T$, horizon $N$
\State \textbf{Output:} Safe robot trajectory $\{x_{\rob,k}\}_{k=0}^N$
\State Sample human trajectories: $\{\hat{u}_{\hum,k}\}_{k=0}^N \sim P(u_{\hum}|\ x_{0:k})$
\State Compute predicted human states $\{\hat{x}_{\hum,k}\}_{k=0}^N$ by integrating human dynamics
\For{$k = 0$ to $N-1$}
    \State Compute $\hat{\mc{B}}_k$ using $(\hat{x}_{\hum,k}, \hat{u}_{\hum,k})$ via \eqref{eq:barrier_certificates}
    \State Extract feature vector $\phi_k$ 
    \State Update safety margin: $\hat{\lambda}_k = \Phi(\phi_k)$
    \State Solve \eqref{eq:crc_safety_filter} to obtain $u_{\rob,k}$
    \State Apply control $u_{\rob,k}$ to the robot system
\EndFor
\State \Return $\{x_{\rob,k}\}_{k=0}^N$
\end{algorithmic}
\end{algorithm}

\noindent \textbf{Computing Optimal Safety Margin.}
We now detail the offline CRC calibration procedure for computing optimal safety margins as outlined in Algorithm~\ref{alg:compute_lambda}. We adopt the geometric weight decay approach from \cite{farinhas2024nonexchangeable} to handle temporal dependencies through re-weighting the loss function, ensuring statistically valid risk control guarantees.

We generate $M$ interaction trajectories by simulating the stochastic human policy with a nominal robot policy $\pi_{\text{nom}}$. In our implementation, $\pi_{\text{nom}}$ solves the robust CBF optimization problem \eqref{eq:safety_filter_main_robust} where $u_{\text{nom}}(x_k)$ is computed using a goal-reaching control Lyapunov function. We process the $M$ trajectories in batches of size $K$, where each batch computes safety margins $\{ \hat{\lambda}_{b,k} \}_{k=0}^{N-1}$. Each safety margin is computed by collecting all barrier certificate data from timesteps $k$ through $N-1$ across the $K$ trajectories in that batch, yielding $n_k = K \times (N-k)$ total samples. We apply exponentially decaying weights $w_i = \rho^{n_k+1-i}$ to emphasize more recent timesteps, then solve \eqref{eq:optimal_lambda} to compute the optimal safety margin. We then use these safety margins as ground truth labels to train a predictive model to update $\lambda$ online.

\noindent \textbf{Online Updates.} 
To enable dynamic safety margin adaptation during deployment, we train a 2-layer LSTM network with 32 hidden units, denoted $\Phi$, that takes as input a feature vector $\phi$ concatenating the robot state $x_{\rob}$, human state $\hat{x}_{\hum}$, Euclidean distance between agents $\|p_{\rob} - p_{\hum}\|$, predicted barrier value $\mc{\hat{B}}$, and timestep $k$. We choose an LSTM to capture temporal dependencies in the interaction sequences, though simpler models such as MLPs could also be used. 

The LSTM is trained using MSE loss between the ground truth safety margins from the offline calibration phase and the predicted values from our model, which achieves low error on the training and validation data. During online deployment, we perform inference at each timestep $k$ to update $\tilde{\lambda}_k = \Phi(\phi_k)$ in \eqref{eq:crc_safety_filter}, allowing the robot to dynamically restrict the set of admissible controls in \eqref{eq:prediction_set} based on the uncertainty in the stochastic human policy.

\section{Experimental Results}
To demonstrate the effectiveness of our method, we evaluate CRC-SF through single-agent and multi-agent experiments designed to answer the following research questions: 

\noindent\textbf{RQ1.} How does CRC-SF compare in terms of safety and performance to baseline methods? 

\noindent\textbf{RQ2.} What impact does the online adaptation of the safety margin have on balancing safety and performance compared to fixed safety margins? 

\noindent\textbf{RQ3.} How does our approach perform in complex multi-agent scenarios where human behavior is modeled using a prediction model learned from real-world pedestrian data? 

\begin{figure}[t]
    \centering
    \hspace{-0.2in}  
    \includegraphics[width=1.0\columnwidth]{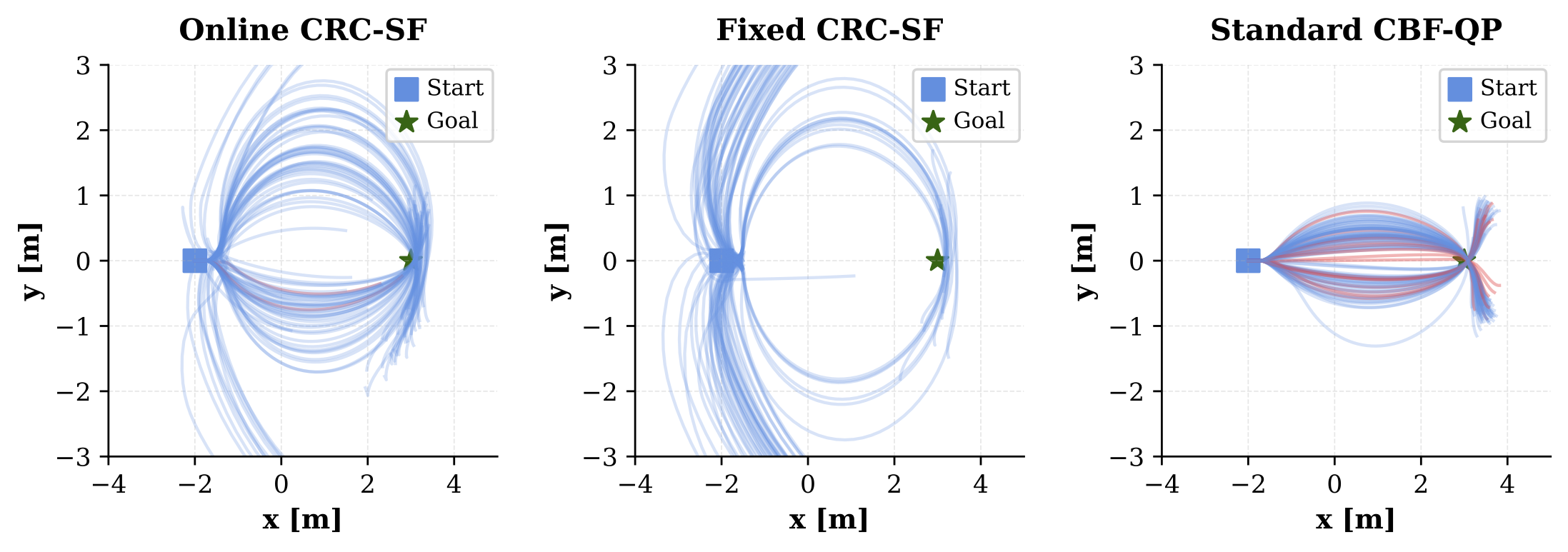}
    \caption{\small Trajectories generated from 100 single-agent head-on interactions with each method. Red trajectories indicate collisions.}
    \label{fig:head_on}
    \vspace{-1em}
\end{figure}

\subsection{Experimental setup}
\noindent \textbf{Choosing the risk threshold $\boldsymbol{\alpha}$.} To demonstrate the application of safety-critical control using our framework, we choose $\alpha = 0.01$ for all of the experiments. According to our theoretical analysis, this should roughly relate to $\gamma = 0.99$, seeking to provide high probabilistic safety guarantees. 

\noindent \textbf{Baselines.} We compare our approach \underline{Online CRC-SF} against three baselines: (i) \underline{CBF-QP} which does not consider uncertainty in human behavior. This baseline provides a direct comparison to show the benefits of incorporating human behavioral uncertainty into the safety constraints. 
(ii) \underline{Fixed CRC-CF} which uses a \textit{fixed} $\lambda$ value computed from considering the entire future trajectory in the calibration phase. This baseline method will help validate the use of an online update scheme to adjust the safety margin to account for different contexts. 
(iii) \underline{MPPI} \cite{Williams2015ModelPP}, a sampling-based method that can handle complex uncertainty distributions, but does not provide formal guarantees.  

\noindent \textbf{Metrics.} We evaluate performance using five key metrics: 
(i) \underline{collision rate}, measuring the percentage of simulations where the robot collides with a human; 
(ii) \underline{safety constraint violation rate}, measuring instances when the robot violates the safe control constraint but not necessarily resulting in collision; 
(iii) \underline{goal-reaching success rate}, capturing the percentage of simulations where the robot reaches within a small neighborhood of its target position within the allotted time horizon; 
(iv) \underline{control effort}, quantifying the magnitude of control inputs applied during the trajectory execution; and 
(v) \underline{control smoothness}, assessing the rate of change in control inputs over time. 

\subsection{Single-agent scenario}
We study a challenging head-on interaction between a human and robot, illustrated in Figure \ref{fig:head_on}. The human randomly chooses to pass either left or right of the robot, requiring the robot to adapt its decisions in real-time. The scenario mimics a corridor situation where the robot must pass the human safely while minimizing deviation from its path to efficiently reach the opposite side.

\noindent \textbf{Simulation setup.} Both the robot and human are modeled using the dynamically-extended unicycle model with state, $x = [p_x, p_y, \theta, v]^\top$ where $(p_x, p_y)$ represents the 2D position, $\theta$ is the heading angle, and $v$ is the velocity. The control inputs $u = [u_1, u_2]^\top$ represent steering and acceleration, bounded by $u \in [-0.3, 0.3] \times [-1, 1]$. The safety constraint is defined using a barrier function that maintains a minimum squared distance between the agents $h(x_\rob, x_\hum) = \|p_\rob - p_\hum\|_2^2 - R^2$ where $R = 1.0$m is the relative safety radius. The class-$\mc{K}_\infty$ function is a linear function $\mc{K}_{\mathrm{CBF}}(r) =  r$.

We set $N = 8.0$s (80 steps at $\Delta T = 0.1$). The human's behavior is simulated using a stochastic policy that generates noisy control inputs around a nominal goal-reaching trajectory, with bounded Gaussian noise ($\sigma = 1.2$, clipped to $\pm 0.5$) added to create noticeable uncertainty in the human's future states. We ran 100 test instances of the simulation.


\input{results_combined}


\noindent\textbf{Results.}
Metrics from this single-agent experiment are summarized in Table~\ref{tab:combined_results_comparison} (bottom).
This single-agent experiment offers insight into \textbf{RQ1} and \textbf{RQ2}.
We observe that our Online CRC-SF achieves the best safety-efficiency trade-off. CBF-QP and Fixed CRC-SF, are either very efficient but not safe, or vice versa. 
Both CRC-SF approaches perform well regarding the safety metrics, which is to be expected, but for Fixed CRC-SF, we see that it is very conservative where safety violations are close to zero percent, but has the worst efficiency by a significantly large margin.
While Online CRC-SF offers a ``close-second'' across all metrics; its safety and performance are significantly more similar to the best performing method in each metric than the worst.

\subsection{Multi-agent scenario}
We evaluate our approach using real-world pedestrian trajectory data recorded from crowd interactions \cite{PellegriniEssEtAl2009}.

\noindent \textbf{Human Behavior Model.} We adapt prior work \cite{MizutaLeung2024} to develop a diffusion model trained on real-world pedestrian trajectory datasets. This enables us to sample multiple possible human action sequences conditioned on the interaction history. Human states are then obtained by integrating these predicted controls using known single-integrator dynamics \cite{Schmerling2017MultimodalPM,SalzmannIvanovicEtAl2020}. This probabilistic formulation allows the robot to reason about the distribution of possible future human states and actions, which is essential for our safety-critical control approach. By sampling from this learned distribution, we can evaluate the probabilistic safety constraints and, crucially, dynamically adapt the robot's level of conservativeness based on the uncertainty in human behavior predictions.

\noindent\textbf{CBF construction.} 
Rather than enforcing constraints with respect to every human in the environment, which could be computationally intractable, we define a neighborhood around the robot within which safety guarantees must be maintained.
Given a robot state $x_{\rob}$, we consider the set of humans $\mc{N}(x_{\rob}) = \{x_{\hum}^i : \|p_{\rob} - p_{\hum}^i\| \leq R_\mc{N}\}$ where $R_\mc{N}=3$m is the neighborhood radius. For each human $i$ in this neighborhood, we construct a corresponding joint barrier function:
$h_i(x_{\rob}, x_{\hum}^i) = \|p_{\rob} - p_{\hum}^i\|_2^2 - R^2$.

\noindent\textbf{Test scenarios.}
We performed experiments on five different scenarios with different initial and goal positions for the robot.
For each scenario, we ran 100 different trials to account for the stochastic nature of the prediction model. 
An example of a multi-agent scenario is shown in Figure~\ref{fig:results_comparison}.

\subsection{Discussion}

\begin{figure}[t]
    \centering
    \begin{subfigure}[b]{0.49\linewidth}
        \centering
        \includegraphics[width=\linewidth]{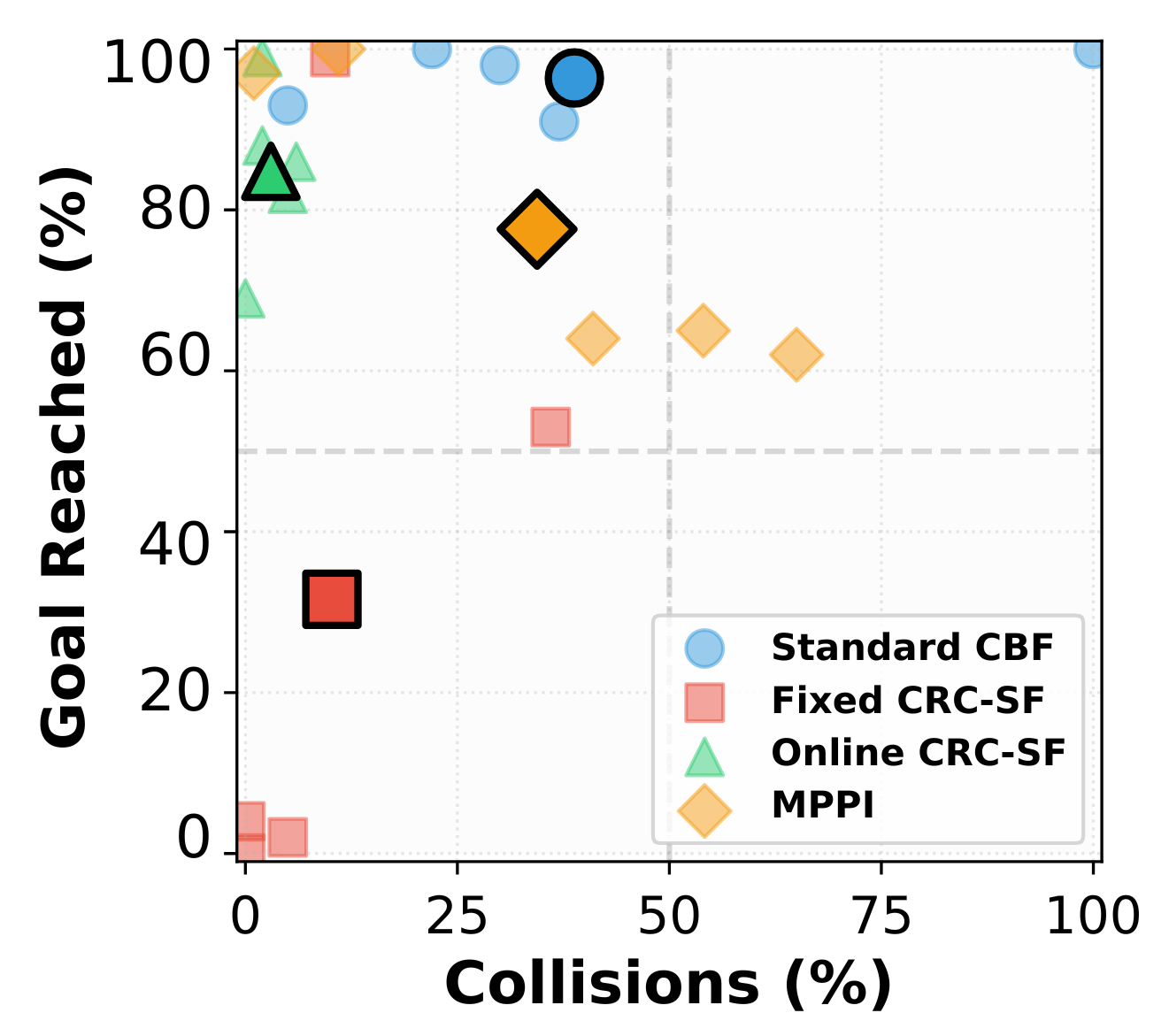}
    \end{subfigure}
    \hfill
    \begin{subfigure}[b]{0.49\linewidth}
        \centering
        \includegraphics[width=\linewidth]{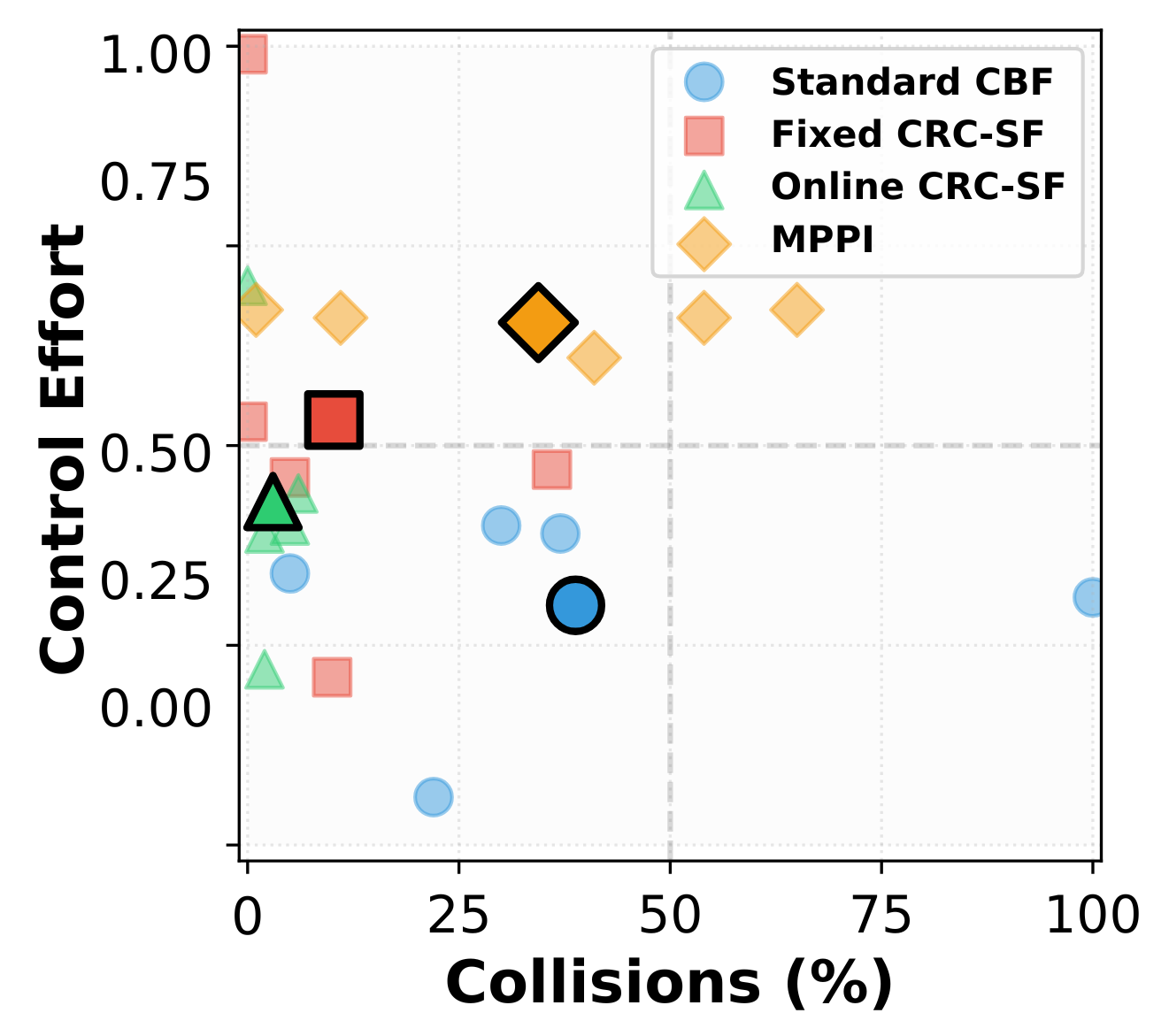}
    \end{subfigure}
    \caption{\small Efficiency versus safety results across five test scenarios for the multi-agent setting. Darker markers with a black outline indicate the mean value across the five test scenarios. Online CRC-SF achieves the best safety-efficiency trade-off while maintaining more consistent performance across scenarios.}
    \label{fig:combined_results_plot}
    \vspace{-1.5em}
\end{figure}

\noindent\textbf{Qualitative results.}
Figure~\ref{fig:results_comparison} visualizes the trajectories from Online CRC-SF and CBF-QP.
We observe that since Online CRC-SF method adapts its safety margin based on context, it proactively ``slows down and waits'' for a pedestrian to pass as it starts to turn and move perpendicularly to the direction of motion (i.e., towards the bottom of the plot). However, CBF-QP, which does not account for uncertainty, ``boldly cuts across'' collides with an oncoming pedestrian. The distance to the closest human over these interactions is shown in Figure~\ref{fig:results_comparison} (right). While Online CRC-SF occasionally violates the safety radius, it consistently recovers without collisions, unlike CBF-QP—suggesting future work to strengthen guarantees during temporary constraint violations.

\noindent\textbf{Quantitative results.}
Metrics from the multi-agent experiments are summarized in Table \ref{tab:combined_results_comparison} (top).
In answering \textbf{RQ3}, we see that even with this more challenging multi-agent setting and the use of a diffusion-based human behavior prediction model, we find that Online CRC-SF performs either the best or a close-second. These results reaffirm the \textbf{RQ1} and \textbf{RQ2} takeaways from the single-agent experiments.
Figure~\ref{fig:combined_results_plot} visualizes this safety-efficiency trade-off for each of the five test scenarios. Not only does Online CRC-SF offer the most desirable safety-efficiency trade-off, but its performance is also consistent across the five test scenarios whereas other methods exhibited higher variance.

\section{Conclusions}
We propose a safe probabilistic robot control framework for interacting with multiple human agents. The proposed safe controller can dynamically adjust its conservativeness based on the uncertainty on future human behaviors. 
We achieve this by leveraging techniques from CRC to form an online estimate of a safety margin for a CBF constraint that will achieve the desired probabilistic safe constraint. 
The experiments on synthetic and real-world interaction data indicate that the approach achieves the best balance between safety and efficiency compared to existing CBF methods. Moreover, the dynamic adjustment of the safety margin naturally induces proactive behaviors in situations with more uncertainty. Further discussion on limitations and future work is contained in Appendix~\ref{app:future_work}.

\bibliographystyle{IEEEtran}
\bibliography{references}

\onecolumn
\clearpage
\appendix
\section{Appendix}
\subsection{Limitations of the Work and Future Directions}
\label{app:future_work}
While Online CRC-SF demonstrates strong performance in the multi-agent navigation setting, there are some limitations to our approach.
First, both CRC-SF approaches will suffer if there is a distribution shift during online deployment. This is because several components rely on an offline dataset, including computing the optimal safety margin, the behavior prediction model, and for Online CRC-SF specifically, training the online safety margin estimation model. While the expectation is that the training data has sufficient coverage, it is not always possible in practice. As such, future work aims to account for distribution shift in the CRC formulation.
Second, as we learn a neural network model to estimate the safety margin parameter online, the estimation model is not perfect. We see this in our results where the robot does not exactly meet the specified $99\%$ probabilistic safe guarantee. Thus in practice, we cannot fully guarantee theoretical safety of our system. Instead, we sought to demonstrate this principled method empirically, and future work will include further testing on more diverse settings and applications, and also improving the estimation model and quantifying the estimation error.

\subsection{Proof of Lemma~\ref{lem:robust_cbf_lemma}}
\label{subsec:proof_robust_cbf}
\robustcbflemma*
\begin{proof}
 To show this holds, if suffices to show that that $u_k$ remains in $\mc{C}(x(t))$ for all $t\in [t_k,t_{k+1}]$. 
    By assumption we have that $x(t_k)\in \mc{S}$ which is equivalent to $h(x(t_k))\geq 0$.     Also by assumption we have that $u(t_k)= u_k\in \widetilde{\mc{C}}(x(t_k))\subseteq\mc{C}(x(t_k))$.
Let us start by showing that for any $\delta\in[0,\Delta T]$, the control input is in $\mc{C}(x(t_k+\delta))$. Using the Lipschitz continuity of the Lie derivatives, the solution to the dynamics, and the bound on the control inputs, we have that
\begin{align*}
        &L_fh(x(t_k+\delta))-L_fh(x(t_k))+\left(L_gh(x(t_k+\delta))-L_gh(x(t_k))\right)u_k\\
        &\quad\geq -\Lip_{h}\Lip_x (\Lip_f+\Lip_{g}\|u_k\|)\cdot \delta\\
        &\quad\geq -\Lip_{h}\Lip_x (\Lip_f+\Lip_{g}B_u)\cdot \delta.
\end{align*}

Rearranging, we have that
        \begin{align*}
        L_fh(x(t_k+\delta))+L_gh(x(t_k+\delta))u_k &\geq 
      L_fh(x(t_k))+L_gh(x(t_k))-\Lip_{h}\Lip_x (\Lip_f+\Lip_{g}B_u)\cdot \delta.      
    \end{align*}
Adding and subtracting $\Lip_{\alpha\circ h}\Lip_x\delta$ on the right hand side of the inequality, we have that 
            \begin{align*}
        L_fh(x(t_k+\delta))+L_gh(x(t_k+\delta))u_k&\geq 
      L_fh(x(t_k))+L_gh(x(t_k))-\Lip_{h}\Lip_x (\Lip_f+\Lip_{g}B_u+\Lip_{\mc{K}\circ h})\cdot \delta+ \Lip_{\mc{K}\circ h}\Lip_x\delta   \\
      &\geq -\mc{K}(h(x(t_k))+ \Lip_{\mc{K}\circ h}\Lip_x\delta.
    \end{align*}
    Since $\mc{K}\circ h$ is Lipschitz continuous, we also have that
    \[|\mc{K}(h(x(t_k+\delta))-\mc{K}(h(x(t_k))|\leq \Lip_{\mc{K}\circ h}\Lip_x\delta\;\Leftrightarrow\; -\Lip_{\mc{K}\circ h}\Lip_x\delta\leq \mc{K}(h(x(t_k+\delta))-\mc{K}(h(x(t_k))\leq \Lip_{\mc{K}\circ h}\Lip_x\delta.\]
    Therefore, by adding and subtracting $\mc{K}(h(x(t_k+\delta))$ on the right hand side of the above inequality, we deduce that
               \begin{align*}
        L_fh(x(t_k+\delta))+L_gh(x(t_k+\delta))u_k 
      &\geq \mc{K}(h(x(t_k+\delta))-\mc{K}(h(x(t_k))+ \Lip_{\mc{K}\circ h}\Lip_x\delta- \mc{K}(h(x(t_k+\delta))\\
      &\geq - \mc{K}(h(x(t_k+\delta)).
    \end{align*} 
    This shows for any $\delta\in [0,\Delta T]$, that we have $u_k\in \mc{C}(x(t_k+\delta))$ and therefore $x(t_k+\delta)\in \mc{S}$ as a consequence by the standard continuous time CBF arguments. 
\end{proof}

\subsection{Proof of Lemma~\ref{lem:crc_concentration}}
\label{subsec:proof_crc_concentration}
\barriervalueconcentration*

\begin{proof}
First, observe that by the definition of our loss function $L(\hat{\lambda})$ in equation \eqref{eq:crc_loss}:
\[
   L(\hat{\lambda}) = \max\left(0, |\mc{B}_k - \mc{\hat{B}}_k| - \hat{\lambda}\right).
\]
For any $\epsilon > 0$, if the barrier prediction error exceeds $\hat{\lambda} + \epsilon$, then the loss must be greater than $\epsilon$:
\[
   |\mc{B}_k - \mc{\hat{B}}_k| > \hat{\lambda} + \epsilon \implies L(\hat{\lambda}) > \epsilon.
\]
This implication allows us to bound the probability of large prediction errors as follows: 
\[
   \Pr(|\mc{B}_k - \mc{\hat{B}}_k| > \hat{\lambda} + \epsilon) = \Pr(L(\hat{\lambda}) > \epsilon).
\]

Recall Markov's concentration inequality: for any non-negative random variable $X$ and constant $a > 0$, 
\[
   \Pr(X > a) \leq \frac{\mathbb{E}[X]}{a}.
\]
Applying Markov's inequality to the loss function and using the non-exchangeable CRC guarantee that $\mathbb{E}[L(\hat{\lambda})] \leq \alpha + \beta$. we have that
\[
   \Pr(L(\hat{\lambda}) > \epsilon) \leq \frac{\mathbb{E}[L(\hat{\lambda})]}{\epsilon} \leq \frac{\alpha + \beta}{\epsilon}.
\]
Setting $\epsilon = \frac{\alpha + \beta}{\gamma}$, where $\alpha$ is the user-specified risk threshold, $\beta$ is the total variation penalty term, and $\gamma$ is the user-specified confidence level, we obtain
\[
   \Pr\left(L(\hat{\lambda}) > \frac{\alpha + \beta}{\gamma}\right) \leq \gamma.
\]
Taking the complement of this probability, we have that
\[
   \Pr\left(L(\hat{\lambda}) \leq \frac{\alpha + \beta}{\gamma}\right) \geq 1 - \gamma.
\]
Since $L(\hat{\lambda}) = \max\left(0, |\mc{B}_k - \mc{\hat{B}}_k| - \hat{\lambda}\right)$, we have that
\[
   \Pr\left(|\mc{B}_k - \mc{\hat{B}}_k| - \hat{\lambda} \leq \frac{\alpha + \beta}{\gamma}\right) \geq 1 - \gamma.
\]
Rearranging, we have that
\[
   \Pr\left(|\mc{B}_k - \mc{\hat{B}}_k| \leq \hat{\lambda} + \frac{\alpha + \beta}{\gamma}\right) \geq 1 - \gamma.
\]
Thus, with $\epsilon = (\alpha + \beta)/\gamma$, we conclude that
\[
   \Pr(|\mc{B}_k - \mc{\hat{B}}_k| \leq \hat{\lambda} + \epsilon) \geq 1 - \gamma.
\]
This completes the proof. 
\end{proof}

\subsection{Proof of Theorem~\ref{thm:crc_cbf}}
\label{subsec:proof_crc_cbf}
\crccbfsafety*

\begin{proof}
    From Lemma~\ref{lem:crc_concentration}, we know that with $\epsilon = \frac{\alpha + \beta}{\gamma}$, the following holds:
\[
   \Pr(|\mc{B}_k - \mc{\hat{B}}_k| \leq \hat{\lambda} + \epsilon) \geq 1 - \gamma.
\]
For any robot control action $u_\rob \in \mc{C}_{\lambda}$, by definition of the prediction set in \eqref{eq:prediction_set}, we have that
\[
   \mc{\hat{B}}_k - (\hat{\lambda} + \epsilon) \geq 0.
\]
With probability at least $1-\gamma$, the barrier prediction error is bounded so that
\[
   |\mc{B}_k - \mc{\hat{B}}_k| \leq \hat{\lambda} + \epsilon.
\]
This inequality implies that both of the following inequalities hold simultaneously:
\[
   \mc{B}_k - \mc{\hat{B}}_k \leq \hat{\lambda} + \epsilon \Longleftrightarrow \mc{B}_k \leq \mc{\hat{B}}_k + (\hat{\lambda} + \epsilon),
\]
and 
\[
 \mc{B}_k - \mc{\hat{B}}_k \geq -(\hat{\lambda} + \epsilon) \Longleftrightarrow \mc{B}_k \geq\mc{\hat{B}}_k  -(\hat{\lambda} + \epsilon).
\]
Suppose we select a control input such that 
\[
    \mc{\hat{B}}_k  -(\hat{\lambda} + \epsilon)\geq 0.
\]
Combining this with the lower bound established above, we have that
\[
    \mc{B}_k \geq \mc{\hat{B}}_k - (\hat{\lambda} + \epsilon) \geq 0.
\]
By the properties of barrier certificates and Lemma~\ref{lem:prob_safety_guarantee}, we know that
\[
   \mc{B}_k \geq 0 \implies h(x_{k+1}) \geq 0.
\]
Therefore, we have the following chain of probabilities:
\begin{align*}
   \Pr(h(x_{k+1}) \geq 0) &\geq \Pr(\mc{B}_k \geq 0), \\
   &\geq \Pr(|\mc{B}_k - \mc{\hat{B}}_k| \leq \hat{\lambda} + \epsilon), \\
   &\geq 1 - \gamma.
\end{align*}
This establishes the desired probabilistic safety guarantee:
\[
   \Pr(h(x_{k+1}) \geq 0) \geq 1 - \gamma.
\]
This completes the proof. 
\end{proof}

\end{document}

%% file: macros.tex
\newcommand{\bmat}[1]{\begin{bmatrix}#1\end{bmatrix}}
\newcommand{\bdd}{\xi}

%% file: results_combined.tex
\begin{table}[t]
\centering
\begin{tabular}{l|ccccc}
\toprule
\multicolumn{6}{c}{\textbf{Multi-Agent Scenario}} \\
\midrule
\textbf{Method} & \begin{tabular}[c]{@{}c@{}}Coll.\\ (\%)↓\end{tabular} & \begin{tabular}[c]{@{}c@{}}Safety\\ Viol.(\%)↓\end{tabular} & \begin{tabular}[c]{@{}c@{}}Goal\\ (\%)↑\end{tabular} & \begin{tabular}[c]{@{}c@{}}Control\\ Effort↓\end{tabular} & \begin{tabular}[c]{@{}c@{}}Control\\ Smooth.↓\end{tabular} \\
\midrule
CBF-QP & 38.8 & 81.4 & 96.4 & 0.30 & 0.12 \\
Fixed CRC-SF & 10.2 & 49.0 & 31.6 & 0.53 & 0.23 \\
\rowcolor{lightblue}Online CRC-SF & 3.0 & 53.2 & 84.8 & 0.43 & 0.19 \\
MPPI & 34.4 & 81.2 & 77.6 & 0.65 & 0.46 \\
\midrule
\multicolumn{6}{c}{\textbf{Single-Agent Scenario}} \\
\midrule
CBF-QP & 16.0 & 57.0 & 100.0 & 0.31 & 0.12 \\
Fixed CRC-SF & 0.0 & 1.0 & 14.0 & 1.35 & 0.30 \\
\rowcolor{lightblue}Online CRC-SF & 2.0 & 15.0 & 78.0 & 0.59 & 0.20 \\
\bottomrule
\end{tabular}
\caption{\small Performance comparison between multi-agent and single-agent human-robot navigation scenarios. The multi-agent results represent averages across five different test configurations (each with distinct initial and goal positions), with each configuration run 100 times. The single-agent results represent averages from 100 head-on simulations.}
\label{tab:combined_results_comparison}
\vspace{-2em}
\end{table}